\documentclass[11pt,onecolumn,draftclsnofoot]{IEEEtran} 
\usepackage{fullpage}

\usepackage{vmargin} 
\usepackage{amsmath}

\setpapersize{A4}
\setmargins{3.5cm}
					 {1.5cm}
           {14.7cm}
           {23.42cm}
           {14pt}
           {1cm}
           {0pt}
           {2cm}

\usepackage{t1enc} 
\usepackage{times}		

\usepackage{graphicx} 
\usepackage{color} 
\usepackage{eso-pic} 
\usepackage{everyshi} 

\usepackage{hyperref}

\usepackage{booktabs}	  	

\usepackage{ccaption} 
\captionnamefont{\bf\footnotesize\sffamily} 
\captiontitlefont{\footnotesize\sffamily} 
\usepackage{scrpage2} 
\usepackage{comment}
\usepackage{multirow}
\renewpagestyle{plain}%
	{(\textwidth,0pt)%
		{\hfill}{\hfill}{\hfill}%
	(\textwidth,0pt)}%
	{(\textwidth,0pt)%
		{\hfill}{\hfill}{\hfill}%
	(\textwidth,0pt)}

\newtheorem{theorem}{Theorem}[section]
\newtheorem{lemma}[theorem]{Lemma}

\newenvironment{proof}[1][Proof]{\begin{trivlist}
\item[\hskip \labelsep {\bfseries #1}]}{\end{trivlist}}

\newcommand{\qed}{\nobreak \ifvmode \relax \else
      \ifdim\lastskip<1.5em \hskip-\lastskip
      \hskip1.5em plus0em minus0.5em \fi \nobreak
      \vrule height0.75em width0.5em depth0.25em\fi}
      
\title{Comments on the proof of adaptive submodular function minimization}
\author{Feng Nan, Venkatesh Saligrama \\ Boston University, Boston MA 02215}
\begin{document}
\maketitle

\begin{abstract}
We point out an issue with Theorem~5 appearing in \cite{GroupBasedActiveLearning}. Theorem~5 bounds the expected number of queries for a greedy algorithm to identify the class of an item within a constant factor of optimal. The Theorem is based on correctness of a result on minimization of adaptive submodular functions. We present an example that shows that a critical step in Theorem~A.11 of \cite{AdaSubmodular_jair2011} is incorrect. 
\end{abstract}

A typical application of the adaptive greedy algorithm is in the disease diagnosis problem: given a set of patients (realizations), each having a specific disease (class), we have access to a set of medical tests (items), each of which produces a discrete outcome (observation) when applied to a patient. Each test incurs a cost. Suppose an unknown patient arrives according to a given distribution among the set of patients, the goal of the problem is to design an adaptive testing strategy so that the expected cost to diagnose the unknown patient is minimized. 
To be concrete and simplify the notation, consider the following numerical example. 

\section*{A Numerical Example}

There are 5 realizations $\phi_1,\dots,\phi_5$, each belonging to a distinct class and 3 items $e_1,e_2,e_3$ of equal cost. Suppose each realization has equal probability mass: $p(\phi_i)=\frac{1}{5},\forall i=1,\dots,5$. The observations can be summarized in a table:
\begin{center}
    \begin{tabular}{| l | l | l | l |}
    \hline
     & $e_1$ & $e_2$ & $e_3$ \\ \hline
    $\phi_1$ & 0 & 1 & 0  \\ \hline
    $\phi_2$ & 0 & 0 & 1 \\ \hline
    $\phi_3$ & 0 & 1 & 1 \\ \hline
    $\phi_4$ & 1 & 0 & 1 \\ \hline
    $\phi_5$ & 1 & 1 & 0 \\ \hline
    \end{tabular}
\end{center}
Let $\pi$ be a policy, which is a decision tree that determines which item to choose based on the observed outcomes of the previous items until a class is determined for a given realization. We identify a node $a$ in the decision tree with all the realizations in it - those realizations that follow the same path according to $\pi$ until $a$. Let $Q_{\mathcal{A}}$ be the set of items chosen according to $\pi$ before node $a$. 
Let the reward function be
\begin{equation*}
f(Q_{\mathcal{A}}, \phi_i)=1-p_a^2+(p_a^{k_i})^2,
\end{equation*}
 where $p_a=\sum_{\phi_i\in a} p(\phi_i)$ is the probability mass of realizations in node $a$; $k_i$ is the class of $\phi_i$ and $p_a^{k_i}$ is the probability mass of realizations that reached node $a$ and are of class $k_i$. 
The above function is shown to be adaptive submodular and strongly adaptive monotone in Lemma 2 and 3 of \cite{GroupBasedActiveLearning}.
We also define the expected reward at node $a$ of the decision tree as
\begin{equation*}
f_E(a)=\sum_{\phi_i\in a} \frac{p(\phi_i)}{p_a}f(Q_{\mathcal{A}},\phi_i).
\end{equation*}
Given policy $\pi$ and a realization $\phi_i$, we can trace a unique path in the decision tree followed by $\phi_i$. Setting a threshold value $x$, we then define a stop node $\psi_\pi(\phi_i,x)$ along the path to be the farthest node from the root for which the expected reward is less than $x$. Formally,
\begin{equation*}
\psi_\pi(\phi_i,x)=\underset{a:\phi_i \in a, f_E(a)<x}{\text{argmax}} f_E(a).
\end{equation*}
Note $f_E$ monotonically increases as nodes move away from the root because $f_E$ is strongly adaptive monotone. In this example we set $x=\frac{23}{25}$. 
The authors in \cite{adaptiveSubmodular_arxiv} claim that the collection of stop nodes $\{\psi_\pi(\phi_i,x),\forall i \}$ form a partition of the realizations. We will show it is not the case.

The adaptive greedy policy will first choose $e_1$. 
Pictorially, the greedy policy can be represented as the decision tree in Figure \ref*{fig}. At the root node $r$ $e_1$ is chosen and node $b$ corresponds to the realizations that have response 0 for $e_1$: $\phi_1,\phi_2,\phi_3$. We can compute the expected reward at $b$ with $p_b=\frac{3}{5}$ as 
\begin{equation*}
f_E(b)= 1-p_b^2+\sum_{\phi_i\in b}\frac{p(\phi_i)}{p_b}(p(\phi_i))^2=1-(\frac{3}{5})^2+\frac{3(1/5)^3}{(3/5)}=\frac{17}{25}<x.
\end{equation*}
Then $e_2$ is chosen at $b$, separating $\phi_1,\phi_3$ to node $c$ and $\phi_2$ to node $d$. Compute the expected reward at $c$ with $p_c=\frac{2}{5}$ as 
\begin{equation*}
f_E(c)= 1-p_c^2+\sum_{\phi_i\in c}\frac{p(\phi_i)}{p_c}(p(\phi_i))^2=1-(\frac{2}{5})^2+\frac{2(1/5)^3}{(2/5)}=\frac{22}{25}<x,
\end{equation*}
and similarly the expected reward at $d$ with $p_d=\frac{1}{5}$ as
\begin{equation*}
f_E(d)= 1-p_d^2+\sum_{\phi_i\in d}\frac{p(\phi_i)}{p_d}(p(\phi_i))^2=1-(\frac{1}{5})^2+(\frac{1}{5})^2=1>x.
\end{equation*}
We also have $f_E(e)=1$ and $f_E(f)=1$. 
\begin{figure}
\centering
 \includegraphics[trim=5cm 8cm 5cm 3cm,angle=0,width=.8\textwidth]{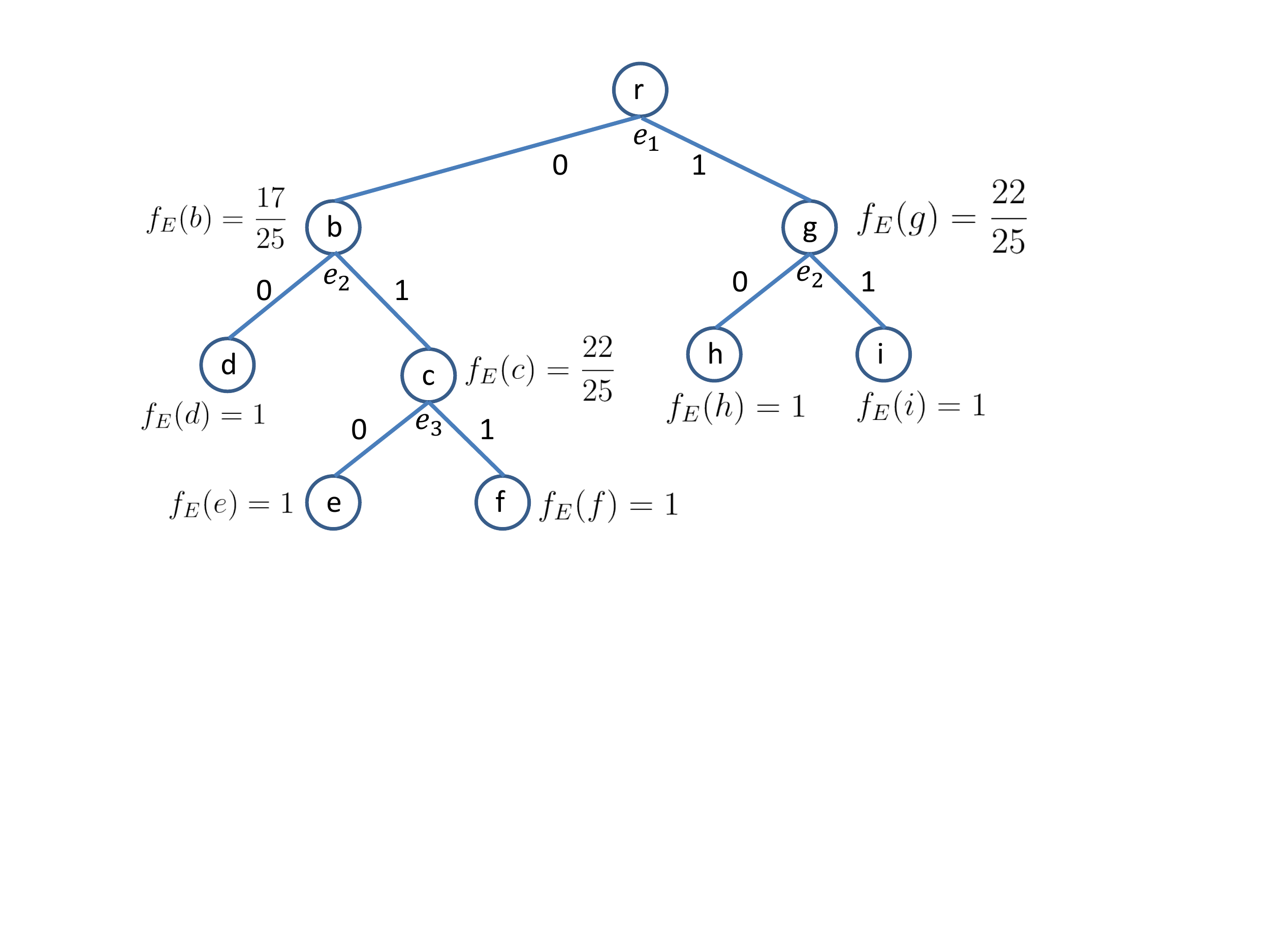}
\caption{Decision tree of the counter example}\label{fig}
\end{figure}

So the stop node for $\phi_2$ is node $b$: $\psi_\pi(\phi_2,x)=b$ and the stop node for $\phi_1$ is node $c$: $\psi_\pi(\phi_1,x)=c$. Clearly nodes $b$ and $c$ contain common realizations  hence they can not be part of a partition of the realizations. 

\section*{Mapping The Notations}
The notations in the above example can be translated to those in \cite{adaptiveSubmodular_arxiv}. Specifically, the realization ($\phi_i$), items ($e_i$) are the same. The expected reward at a node $a$ in a decision tree
is $f_E(a)=\mathbf{E}[f(\text{dom}(\psi),\Phi)|\Phi \sim \psi]$, where $\psi$ is the partial realization that includes the item-observation pairs from the root to node $a$ of the decision tree. Finally the stop node definition of  $\psi_\pi(\phi_i,x)$ in the above example is the same as $\psi(\phi_i,x)$ in the proof of Theorem 37 in \cite{adaptiveSubmodular_arxiv}, except denoting the policy $\pi$ explicitly. 

In the proof of Theorem 37 in \cite{adaptiveSubmodular_arxiv}, before Equation 37, the authors claimed that $\{ \{\phi:\phi \sim \psi_i^x\}: i=1,2,\dots,r \}$ partitions the set of realizations. We showed in our numerical example that it does not necessarily partition the set of realizations. The implication of this is that between Equation 37 to 38, "=" should be replaced with "$\geq$" because of overcounting. 
To show that Equation 37 is not equal to Equation 38 in \cite{adaptiveSubmodular_arxiv} based on our numerical example, let $c_1=c(\pi^*_{\text{avg}}|\phi_1), \dots, c_5=c(\pi^*_{\text{avg}}|\phi_5)$. Then we have 
\begin{align*}
&\mathbf{E}[c(\pi^*_{\text{avg}}|\psi_\pi(\Phi,x))]\\
&=p_b\sum_{\phi_i\in b}\frac{p(\phi_i)}{p_b}c_i+p_c\sum_{\phi_i\in c}\frac{p(\phi_i)}{p_c}c_i+p_g\sum_{\phi_i\in g}\frac{p(\phi_i)}{p_g}c_i\\
&=\frac{3}{5}\cdot \frac{1}{3}(c_1+c_2+c_3)+\frac{2}{5}\cdot \frac{1}{2}(c_1+c_3)+\frac{2}{5}\cdot \frac{1}{2}(c_4+c_5)\\
&=\frac{1}{5}(c_1+c_2+c_3+c_4+c_5)+\frac{1}{5}(c_1+c_3)\\
&> \frac{1}{5}(c_1+c_2+c_3+c_4+c_5)=c_{\text{avg}}(\pi^*_{\text{avg}}).
\end{align*}
This breaks the chain of inequalities used later to prove the final cost bound.

\section*{Appendix: The Theorem in Question}
We copy the theorem in question below for easy reference.
\begin{theorem}[Theorem 37 in \cite{adaptiveSubmodular_arxiv}]
Suppose $f: 2^E \times O^E \to \mathbf{R}_{\geq 0}$ is adaptive submodular and strongly adaptive monotone with respect to $p(\phi)$ and there exists $Q$ such that $f(E,\phi)=Q$ for all $\phi$. Let $
\eta$ be any value such that $f(S,\phi)>Q-\eta$ implies $f(S,\phi)=Q$ for all $S$ and $\phi$. Let $\delta=min_\phi p(\phi)$ be the minimum probability of any realization. Let $\pi_{avg}^*$ be an optimal policy minimizing the expected cost of items selected to guarantee every realization is covered. Let $\pi$ be an $\alpha$-approximate greedy policy with respect to the item costs. Then in general 
\begin{equation*}
c_{avg}(\pi) \leq \alpha c_{avg}(\pi_{avg}^*)(\ln (\frac{Q}{\delta \eta})+1)
\end{equation*}
and for self-certifying instances
\begin{equation*}
c_{avg}(\pi) \leq \alpha c_{avg}(\pi_{avg}^*)(\ln (\frac{Q}{\eta})+1).
\end{equation*}
\end{theorem}

\bibliographystyle{plain}
\bibliography{OptimalTree}
\end{document}